\documentclass[letterpaper, 10 pt, conference]{ieeeconf}  % Comment this line out if you need a4paper

\IEEEoverridecommandlockouts                              % This command is only needed if 
\usepackage{amsmath} % assumes amsmath package installed
\usepackage{amssymb}  % assumes amsmath package installed

\usepackage{graphicx}
\usepackage{cite}
\usepackage{caption}
\usepackage{subcaption}
\usepackage{multirow,multicol}
\usepackage[linesnumbered,ruled,vlined]{algorithm2e}
\usepackage{setspace}
\SetAlFnt{\footnotesize}
\SetAlCapFnt{\normalfont\footnotesize}
\SetAlCapNameFnt{\footnotesize}
\SetArgSty{textrm}
\SetInd{0.1em}{0.5em}
\SetCommentSty{emph}
\usepackage{tikz}

\usepackage{color}

\title{\LARGE \bf
Optimal and Bounded-Suboptimal Multi-Goal Task Assignment \\and Path Finding
}

\author{Xinyi Zhong$^{1}$, Jiaoyang Li$^{2}$, Sven Koenig$^{2}$, Hang Ma$^{1}$% <-this % stops a space
\thanks{$^{1}$School of Computing Science, Simon Fraser University, Canada
        {\tt\small \{xinyi\_zhong,hangma\}@sfu.ca}}%
\thanks{$^{2}$Department of Computer Science, University of Southern California, USA
        {\tt\small \{jiaoyanl,skoenig\}@usc.edu}}%
}

\begin{document}

\maketitle
\thispagestyle{empty}
\pagestyle{empty}

%%%%%%%%%%%%%%%%%%%%%%%%%%%%%%%%%%%%%%%%%%%%%%%%%%%%%%%%%%%%%%%%%%%%%%%%%%%%%%%%
\begin{abstract}
We formalize and study the multi-goal task assignment and path finding (MG-TAPF) problem from theoretical and algorithmic perspectives. The MG-TAPF problem is to compute an assignment of tasks to agents, where each task consists of a sequence of goal locations, and collision-free paths for the agents that visit all goal locations of their assigned tasks in sequence. Theoretically, we prove that the MG-TAPF problem is NP-hard to solve optimally. We present algorithms that build upon algorithmic techniques for the multi-agent path finding problem and solve the MG-TAPF problem optimally and bounded-suboptimally. We experimentally compare these algorithms on a variety of different benchmark domains. 

\end{abstract}

%%%%%%%%%%%%%%%%%%%%%%%%%%%%%%%%%%%%%%%%%%%%%%%%%%%%%%%%%%%%%%%%%%%%%%%%%%%%%%%%

\section{Introduction}

In recent years, the multi-agent path finding (MAPF) problem \cite{SternSOCS19} has been well-studied in artificial intelligence and robotics due to its many applications, such as warehouse automation \cite{kiva}, autonomous traffic management \cite{dresner2008multiagent}, autonomous aircraft towing \cite{airporttug16}, and video games \cite{LiAAMAS20}. In the MAPF problem, each agent must move from its current location to its pre-assigned goal location while avoiding collisions with other agents in a known environment. 

The MAPF problem has recently been extended to many real-world settings \cite{GTAPF, LiuAAMAS19} where goal locations are not pre-assigned to agents. For example, in a modern automated warehouse, each warehouse robot needs to pick up an inventory pod from its storage location, deliver it to the inventory stations that request one or more products stored in it, and take it back to its storage location. Such automated warehouse systems often employ a task planner to determine a set of tasks consisting of a sequence of goal locations. The problem is then to assign these tasks to the warehouse robots and plan paths for them. 

We thus formalize and study the multi-goal task assignment and path finding (MG-TAPF) problem, where as many tasks as agents are given and each task consists of a sequence of goal locations. The MG-TAPF problem is to compute a one-to-one assignment of tasks to agents and plan collision-free paths for the agents from their current locations to the goal locations of their assigned tasks such that each agent visits the goal locations in the order specified by its assigned task and the flowtime (the sum of the finish times of all agents in the last goal locations of their assigned tasks) is minimized.

\subsection{Background and Related Work}

Many problems that are related to our problem have been proposed and studied in recent years.

\noindent \textbf{MAPF:}
The MAPF problem is NP-hard to solve optimally for flowtime (the sum of the finish times of all agents in the last goal locations of their assigned tasks) minimization \cite{YuLav13AAAI} and even NP-hard to approximate within any constant factor less than 4/3 for makespan (the maximum of the finish times of all agents in their pre-assigned goal locations) minimization \cite{surynek2010optimization, MaAAAI16}. MAPF algorithms include reductions to other well-studied optimization problems \cite{YuLav13ICRA, Surynek15, erdem2013general} and specialized rule-based, search-based and hybrid algorithms \cite{PushAndSwap, WangB11, PushAndRotate, DBLP:journals/ai/SharonSGF13, Wagner15, DBLP:journals/ai/SharonSFS15, LamBHS19}. In particular, Conflict-Based Search (CBS) \cite{DBLP:journals/ai/SharonSFS15} is a popular two-level optimal MAPF algorithm that computes time-optimal paths for individual agents on the low level and performs a best-first tree search to resolve collisions among the paths on the high level. Recent research has developed several improved versions of CBS \cite{ICBS, LiIJCAI19, GangeHS19, DBLP:conf/ijcai/BoyarskiFHS00K20}. An extended version of CBS has been developed for the case where each agent has multiple pre-assigned goal locations \cite{SurynekAAAI21}. However, the MAPF problem is insufficient for modelling real-world settings where goal locations are not pre-assigned to agents.

\noindent \textbf{TAPF:}
In the TAPF problem \cite{MaAAMAS16}, agents are partitioned into teams and each team is given the same number of single-goal tasks as the number of agents. The objective is to assign each agent of a team exactly one task of the team and plan collision-free paths for the agents. A special case of the TAPF problem occurs when only one team exists, known as the anonymous MAPF problem. We use TAPF to denote this special case. TAPF can be solved optimally in polynomial time for makespan minimization\cite{YuLav13STAR}. For flowtime minimization, its complexity remains unclear, but it can be solved with an extension of CBS, called CBS with Task Assignment (CBS-TA)\cite{honig2018conflict}. CBS-TA searches in the space of all possible assignments of tasks to agents using a best-first search in a forest that consists of regular CBS search trees where each tree corresponds to a different assignment. A similar version of CBS has been developed for the case where each agent can execute more than one task, but it scales to only four agents and four tasks \cite{DBLP:conf/iros/HenkelAT19}.

\noindent \textbf{Lifelong TAPF:}
Recent research has considered the online multi-agent pickup and delivery (MAPD) problem, where tasks appear at unknown times and each task consists of a sequence of two goal locations \cite{MaAAMAS17}. The offline MAPD problem has also been studied for the case where all tasks are known in advance \cite{LiuAAMAS19}. The multi-goal TAPF (MG-TAPF) problem is at the crux of two lifelong problems since state-of-the-art MAPD algorithms \cite{LiuAAMAS19, grenouilleau2019multi} essentially decompose a lifelong instance into a sequence of MG-TAPF instances. However, the MG-TAPF instances are not solved optimally in either case. Grenouilleau et al. \cite{grenouilleau2019multi} have proposed the Multi-Label Space-Time A* algorithm (MLA*), that computes a time-optimal path for an agent that visits two goal locations and solves task assignment and path finding independently.  A similar lifelong problem has also been considered where each task has a temporal constraint\cite{GTAPF}.

\noindent \textbf{PC-TAPF:}
The precedence-constraint TAPF (PC-TAPF) problem generalizes the anonymous MAPF problem \cite{brown2020optimal}. It involves sequential task assignment and incorporates temporal precedence constraints between tasks, where for example, tasks A and B must be completed before task C begins. The PC-TAPF problem is NP-hard to solve for makespan minimization. A four-level hierarchical algorithm is proposed to solve it optimally. However, the lower level path planner (third and fourth level) is incomplete. So, there is no completeness guarantee for the whole algorithm.

\subsection{Contributions}

In this paper, we study the general version of the TAPF problem, where each task consists of a sequence of multiple ordered goal locations, called MG-TAPF problem.

%As our first contribution, w
We formalize the MG-TAPF problem as an extension of the TAPF problem that aims to minimize the flowtime. We prove that it is NP-hard to solve optimally. %, even for the case where each task consists of only two goal locations. 
The proof is based on a reduction \cite{MaAAAI16} from a specialized version of the Boolean satisfiability problem \cite{cat1984} to the MG-TAPF problem.

%As our second contribution, w
We present an Conflict-Based Search with Task Assignment with Multi-Label A* algorithm (CBS-TA-MLA) that solves the MG-TAPF problem optimally for flowtime minimization. CBS-TA-MLA is a hierarchical algorithm. It uses a best-first search algorithm CBS-TA \cite{honig2018conflict} on the high level to search over all possible assignments of tasks to agents and resolve collisions among paths and MLA* \cite{grenouilleau2019multi} on the low level to compute a time-optimal path for each agent that visits the goal locations of its assigned task in sequence. We prove that CBS-TA-MLA is correct, complete and optimal. 

%As our third contribution, w
We develop three admissible heuristics for the high-level search of CBS-TA-MLA based on the existing admissible heuristics \cite{LiIJCAI19} for CBS for the MAPF problem and generalize Multi-Valued Decision Diagrams (MDDs) from the case of one goal location to the case of multiple goal locations. We also extend CBS-TA-MLA to a bounded-suboptimal version, called ECBS-TA-MLA, using ideas from the bounded-suboptimal version of CBS \cite{ECBS}. We experimentally compare the proposed algorithms for a variety of benchmark domains.%, showcasing their practicability for many real-world applications.

%%%%%%%%%%%%%%%%%%%%%%%%%%%%%%%%%%%%%%%%%%%%%%%%%%%%%%%%%%%%%%%%%%%%%%%%%%%%%%%%

\section{Problem Definition}

%The MG-TAPF problem consists of $m$ agents $\{a_1, a_2,\ldots,a_m\}$ and an undirected graph $G=(V, E)$, where $V$ is the set of locations, and $E$ is the set of unit-weight edges connecting locations that agents can move along. Each agent $a_i$ has a start location $s_i \in V$. A set of $m$ tasks $\{\boldsymbol g_1, \boldsymbol g_2, \ldots,\boldsymbol g_m\}$ is also given where each task $\boldsymbol g_j$ is characterized by a sequence of $K_j$ goal locations $\boldsymbol g_j = \langle \boldsymbol g_{j}[1], \ldots, \boldsymbol g_{j}[K_j]\rangle$. Each agent $a_i$ can be assigned any task $\boldsymbol g_j$. 

The MG-TAPF problem instance consists of (1) an undirected graph $G=(V, E)$, where $V$ is the set of locations and $E$ is the set of unit-weight edges connecting locations, % that agents can move along, 
(2) $m$ agents $\{a_1, a_2,\ldots,a_m\}$, and for each agent $a_i$, there is a start location $s_i \in V$, and (3) $m$ tasks $\{\boldsymbol g_1, \boldsymbol g_2, \ldots,\boldsymbol g_m\}$, where each task $\boldsymbol g_j$ is characterized by a sequence of $K_j$ goal locations $\boldsymbol g_j = \langle \boldsymbol g_{j}[1], \ldots, \boldsymbol g_{j}[K_j]\rangle$.
Each agent $a_i$ can be assigned any task $\boldsymbol g_j$. 

Let $\pi_i(t)$ denote the location of agent $a_i$ at time $t$. A path $\pi_i = \langle \pi_i(0),\ldots ,\pi_i(T_i),\pi_i(T_i+1), \ldots \rangle$ for agent $a_i$ is a sequence of locations that satisfies the following conditions: (1) The agent starts at its start location, that is $\pi_i(0)=s_i$; (2) At each timestep $t$, the agent either moves to a neighboring location $\pi_i(t+1)\in V$ where $(\pi_i(t), \pi_i(t+1)) \in E$, or stays in its current locations, that is $\pi_i(t)=\pi_i(t+1)$; and (3) The agent visits all goal locations of its assigned task $\boldsymbol g_j$ in sequence and remains in the final goal location at the \textit{finish time} $T_i$, which is the minimum time $T_i$ such that $\pi_i(t)=\boldsymbol g_j[K_j]$ for all times $t=T_i, \ldots, \infty$.

Agents need to avoid collisions while moving to their goal locations. A collision between agents $a_i$ and $a_{j}$ is either: (1) a vertex collision $\langle a_i, a_j, u, t \rangle$, where two agents $a_i$ and $a_j$ are in the same location $u = \pi_i(t)=\pi_{j}(t)$ at time $t$; or (2) an edge collision $\langle a_i, a_j, u, v, t \rangle$ where two agents $a_i$ and $a_j$ traverse the same edge $(u, v)$ in  opposite directions $u = \pi_i(t) = \pi_j (t+1)$ and $v = \pi_i(t+1) = \pi_j(t+1)$ at timestep $t$. A \textit{plan} consists of an assignment of tasks to agents and a path for each agent. A \textit{solution} is a plan whose paths are collision-free. The flowtime $\sum_{i=1}^m T_i$ of a plan is the sum of the finish times of all agents. The problem of MG-TAPF is to find a solution that minimizes the flowtime. In this paper, we only consider the flowtime objective even though many of our results could be easily generalized to other objectives, such as makespan (the maximum of the finish times $\max_{1 \leq i \leq m} T_i$ of all agents) minimization.%, partially because the flowtime objective matches the throughput objective in the lifelong (for example, MAPD) problems.  

%%%%%%%%%%%%%%%%%%%%%%%%%%%%%%%%%%%%%%%%%%%%%%%%%%%%%%%%%%%%%%%%%%%%%%%%%%%%%%%%

\section{Complexity}

We show that the MG-TAPF problem is NP-hard to solve optimally for flowtime minimization, even when each task has only two goal locations. Similar to \cite{MaAAAI16} and \cite{MaAAMAS18}, we use a reduction from $\le$3,$=$3-SAT \cite{cat1984}, an NP-complete version of the Boolean satisfiability problem. A $\le$3,$=$3-SAT instance consists of $N$ Boolean variables $\{X_1, \ldots, X_N\}$ and $M$ disjunctive clauses $\{C_1, \ldots, C_M\}$, where each variable appears in exactly three clauses, uncomplemented at least once, and complemented at least once, and each clause contains at most three literals. Its decision question asks whether there exists a satisfying assignment. We first show a constant-factor inapproximability result for makespan minimization.

\newtheorem{theorem}{Theorem}
\begin{theorem}\label{thm:makespan_hardness}
    For any $\epsilon > 0$, it is NP-hard to find a $(4/3 - \epsilon)$-approximate solution to the MG-TAPF problem for makespan minimization, even if each task has exactly two goal locations.
\end{theorem}

\begin{proof}
    We use a reduction similar to that used in the proof of Theorem 3 in \cite{MaAAAI16} to construct an MG-TAPF instance with $m=M+2N$ agents and the same number of tasks that has a solution with makespan three if and only if a given $\le$3,$=$3-SAT instance with $N$ variables and $M$ clauses is satisfiable.
    
	We follow the notations used in the proof of Theorem 3 in \cite{MaAAAI16} and point out the differences here: For each variable $X_i$ in the $\le$3,$=$3-SAT instance, we construct two ``literal'' agents $a_{iT}$ and $a_{iF}$ with start locations $s_{iT}$ and $s_{iF}$, and two tasks $\boldsymbol g_{iT}$ and $\boldsymbol g_{iF}$, each with two goal locations. We set $\boldsymbol g_{iT}[1] = s_{iT}$, $\boldsymbol g_{iT}[2] = t_{iT}$, $\boldsymbol g_{iF}[1] = s_{iF}$, and $\boldsymbol g_{iF}[2] = t_{iF}$. For each clause $C_j$ in the $\le$3,$=$3-SAT instance, we construct a ``clause'' agent $a_j$ with start location $c_j$ and a task $\boldsymbol g_{j}$ with two goal locations $\boldsymbol g_{j}[1] = c_{j}$ and $\boldsymbol g_{j}[2] = d_{j}$. Therefore, any optimal solution must assign every task to the agent whose start location is the first goal location of the task and let the agent execute the task.
	
	Using the same arguments as in the proof of Theorem 3 in \cite{MaAAAI16}, we can show that the constructed MG-TAPF instance has a solution with makespan three if and only if the $\le$3,$=$3-SAT instance is satisfiable, and always has a solution with makespan four, even if the $\le$3,$=$3-SAT instance is unsatisfiable. For any $\epsilon>0$, any MG-TAPF algorithm with approximation ratio $4/3 - \epsilon$ thus computes a solution with makespan three the $\le$3,$=$3-SAT instance is satisfiable and thus solves $\le$3,$=$3-SAT problem.
\end{proof}

In the proof of Theorem \ref{thm:makespan_hardness}, the MG-TAPF instance reduced from the given $\le$3,$=$3-SAT instance has the property that the length of every path from the start location of every agent to the final goal location of the task assigned to the agent is at least three. Therefore, if the makespan is three, every agent arrives at the final goal location of its assigned task in exactly three timesteps, and the flowtime is $3m$. Moreover, if the makespan exceeds three, the flowtime exceeds $3m$, yielding the following theorem.

\begin{theorem}\label{thm:flowtime_hardness}
    It is NP-hard to find the optimal solution to the MG-TAPF problem for flowtime minimization, even if each task has exactly two goal locations.
\end{theorem}

%%%%%%%%%%%%%%%%%%%%%%%%%%%%%%%%%%%%%%%%%%%%%%%%%%%%%%%%%%%%%%%%%%%%%%%%%%%%%%%%

\section{CBS-TA-MLA}

The CBS-TA-MLA algorithm is a two-level search algorithm, where the low-level MLA* algorithm plans an optimal path for each agent based on the task assignment and the constraints provided by the high-level CBS-TA algorithm.

\subsection{High Level: CBS-TA}

\begin{algorithm}[t]%[th!]
\setstretch{0.7}
\DontPrintSemicolon
\SetKwInOut{input}{Input}
\SetKwInOut{output}{Output}
\SetKw{return}{Return}

OPEN $\gets \emptyset$ \;
\tcp{initialize first root node R}
R.root $\gets$ True \;
R.assignment $\gets$ firstAssignment() \;
R.constraints $\gets \emptyset$ \;
\For{each agent $a_i$}{
    R.paths[$a_i$] $\gets$ MLA*($a_i$, R.assignment[$a_i$], R.constraints) \;
}
R.cost $\gets$ flowtime(R.paths) \;
R.collisions $\gets$ findCollisions(R) \;
OPEN $\gets \text{OPEN} \cup \{R\}$ \;

\While{OPEN $\neq \emptyset$}{
    
    N $\gets$ lowest cost node from OPEN \;
    OPEN $\gets \text{OPEN} \setminus  \{\text{N}\}$ \;
    \If{N.paths do not have collisions}{
        \return{N.assignment, N.paths}
    }
    
    \If{N.root is True}{
        \tcp{initialize new root node R with next-best task assignment}
        R.root $\gets$ True \;
        R.assignment $\gets$ nextAssignment() \;
        R.constraints $\gets \emptyset$ \;
        \For{each agent $a_i$}{
            R.paths[$a_i$] $\gets$ MLA*($a_i$, R.assignment[$a_i$], R.constraints) \;
        }
        R.cost $\gets$ flowtime(R.paths) \;
        R.collisions $\gets$ findCollisions(R) \;
        OPEN $\gets \text{OPEN} \cup \{R\}$ \;
        
    }
    
    $\langle a_i, a_j, u, t \rangle / \langle a_i, a_j, u, v, t \rangle$  $\gets$ chooseCollision(N)  \;

    \tcp{generate child nodes}
    
    \For{agent $a_k$ in $\{a_i, a_j\}$}{
        Q.root $\gets$ False \;
        Q.assignment $\gets$ N.assignment \;
        Q.constraints $\gets \text{N.constraints} \cup  \{\langle a_k, u, t \rangle /  \langle a_k, u, v, t \rangle\}$ \;
        Q.paths[$a_k$] $\gets$ MLA*($a_k$, Q.assignment[$a_k$], Q.constraints) \;
        \If{Q.paths[$a_k$] is None}{
            continue to the next iteration \;
        }
        Q.cost $\gets$ flowtime(Q.paths) \;
        Q.collisions $\gets$ findCollisions(Q) \;
        OPEN $\gets \text{OPEN} \cup \{$Q$\}$ \;
        
    }
}

\return{No Solution}

\caption{High Level of CBS-TA-MLA}\label{alg:CBS-TA}
\end{algorithm}

Conflict-Based Search with Task Assignment (CBS-TA) is a best-first search algorithm, which was initially designed to solve the TAPF problems \cite{honig2018conflict}. We extend it to solving MG-TAPF problem by replacing the low-level search algorithm with MLA*. Algorithm \ref{alg:CBS-TA} shows the pseudo-code. CBS searches a binary tree, called constraint tree (CT), while CBS-TA searches a forest that contains multiple CTs. Each tree corresponds to a different task assignment. Each node, called CT node, in the CT contains (1) a Boolean value $root$, indicating whether the node is a root node of a CT; (2) an $assignment$, which is the task assignment of the node; (3) a set of $constraints$, where a vertex constraint $\langle a_i, u, t \rangle$ prohibits agent $a_i$ from being at location $u$ at time $t$ and an edge constraint $\langle a_i, u, v, t \rangle$ prohibits agent $a_i$ from moving along from $u$ to $v$ at timestep $t$; (4) a set of $paths$ with respect to the task assignment and the constraints; and (5) a $cost$, which is the flowtime of the paths [Lines 2-7]. 

An $m \times m$ cost matrix $C$ stores the distances from the start locations of all agents to the final goal locations of all tasks where all intermediate goal locations of the task are visited in sequence while ignoring collisions with the other agents. CBS-TA starts with a single root node with the best task assignment (the task assignment with the lowest flowtime which ignoring collisions among agents). The best task assignment is calculated by applying the Hungarian method \cite{Kuhn1955} to the cost matrix $C$. Once the task assignment is calculated, the corresponding paths of the agents are planned by the low-level MLA* search algorithm . CBS-TA then finds collisions among the planned paths, stores the number of collisions in the node and adds the node to the OPEN list [Lines 8-9]. A new root node with the next-best task assignment is created if the currently expanded node is a root node [Lines 15-23]. We use the next-best task assignment algorithm in \cite{honig2018conflict}. 
See \cite{honig2018conflict} for details.

CBS-TA removes a node $N$ with the lowest cost $N.cost$ from the OPEN list to expand (breaking ties in favor of the paths in node with the smallest number of collisions) %(tie-breakers are the number of collisions and then the generated time of the node)
[Lines 11-12]. First, it checks whether the number of collisions is 0. If so, $N$ is declared a goal node, and $N.assignment$ and $N.paths$ are returned [Lines 13-14]. Otherwise, CBS-TA resolves an earliest vertex collision $\langle a_i, a_j, u, t\rangle$ (or edge collision $\langle a_i, a_j, u, v, t\rangle$) [Line 24] by generating two child nodes. Child nodes inherit the constraint set and paths from $N$ [Lines 25-29]. CBS-TA adds constraint $\langle a_i, u, t \rangle$ (or $\langle a_i, u, v, t \rangle$) to the constraint set of one child node, and adds constraint $\langle a_j, u, t \rangle$ (or $ \langle a_j, v, u, t \rangle$) to that of the other child node. It then calls the low-level MLA* search algorithm to replan the path of $a_i$ (or $a_j$) to satisfy the new constraint set. If such a path does not exist, CBS-TA prunes the child node [Lines 30-31]. Otherwise, CBS-TA updates the cost and the number of collisions between the newly planned path and the existing paths of the other agents and adds the child node to the OPEN list [Lines 32-34]. Once the OPEN list is empty, CBS-TA terminates the search unsuccessfully [Line 35]. 

\subsection{Low Level: MLA*}

Multi-Label Space-Time A* (MLA*) finds a time-optimal path for an agent $a_i$ (a path with the smallest finish time $T_i$) that visits all goal locations of its assigned task $\boldsymbol g_j$ in sequence and obeys a set of constraints. MLA* was first introduced for two goal locations \cite{grenouilleau2019multi} and then extended to more than two goal locations \cite{li2020lifelong}. MLA* extends Space-Time A* \cite{WHCA} by adding a label indicating the different segments between the goal locations, where label $k$ indicates that the next goal location to visit is $\boldsymbol g_j[k]$. 

We now formally describe MLA*. MLA* is an A* search whose states are tuples of a location, a time and a label. It starts with state $\langle s_i, 0, 1 \rangle$, indicating agent $a_i$ being at location $s_i$ at time 0 with label 1. A directed edge exists from state $\langle u, t, k \rangle$ to state $\langle v, t+1, k' \rangle$ if and only if (1) $u = v$ or $(u, v) \in E$ and (2) $k' = k+1$ if $v = \boldsymbol g_j[k]$ and $k'=k$ otherwise.  To obey the constraints, the set of states $\{\langle v, t, k \rangle \mid k=1, \ldots, K_j+1\}$ is removed from the state space of agent $a_i$ if and only if there is a vertex constraint $\langle a_i, v, t \rangle$, and the set of edges $\{(\langle u, t, k \rangle, \langle v, t+1, k'\rangle) \mid k=1,\ldots,K_j\}$ is removed if and only if there is an edge constraint $\langle a_i, u, v, t \rangle$. If MLA* expands a goal state $\langle \boldsymbol g_j[K_j], t, K_j+1 \rangle$ and the agent can stay at the goal location forever (without violating any vertex constraints), it terminates and returns the path from the start state to the goal state.

During the search, the $h$-value of each state $\langle v, t, k \rangle$ is set to $\text{dist}(v, \boldsymbol g_j[k])+\sum_{k'=k}^{K_j-1} \text{dist}(g_j[k'], g_j[k'+1])$, that is, the shortest distance from location $v$ to visit all unvisited goal locations in task $\boldsymbol g_j$ in sequence. The distances dist($v, \boldsymbol g_j[k]$) from each location $v \in V$ to all goal locations $\boldsymbol g_j[k]$ with $j=1, \ldots, m$ and $k=1, \ldots, K_j$ are pre-computed by searching backward once from each goal location $\boldsymbol g_j[k]$ on graph $G$.

\subsection{Properties of CBS-TA-MLA}
We now show that CBS-TA-MLA is complete and optimal.

\begin{theorem}
    CBS-TA-MLA is guaranteed to find an optimal solution if the given MG-TAPF instance is solvable and correctly identifies an unsolvable MG-TAPF instance with an upper bound of $\mathcal O(|V|^3 \cdot \sum_{j=1}^m K_j)$ on the finish time $T_i$ of any agent at the final goal location of its assigned task.
\end{theorem}

\begin{proof}
The proof of the optimality of  CBS-TA-MLA is trivial as CBS-TA and MLA* have been proved to be optimal in \cite{honig2018conflict} and \cite{grenouilleau2019multi}, respectively.
As for the completeness, consider an arbitrary optimal solution to the given MG-TAPF instance with paths $\pi_i$. The solution can be divided chronologically into at most $\mathcal K = \sum_{j=1}^m K_j$ segments at breakpoints $t^{(0)}=0, t^{(1)}, \ldots, t^{\mathcal K} = \max_{i} T_i$ where the label of at least one agent changes (because it reaches a new goal location of its assigned task) at each agent $t^{(\kappa)}$.
Since there exists a solution with at most $U=\mathcal O(|V|^3)$ ($U$ is a constant depending on $V$ only) agent movements (edge traversals) to any solvable MAPF instance \cite{YuR14}, there exist collision-free paths for all agents with makespan at most $U$ that move each agent $a_i$ from $\pi(t^{(\kappa-1)})$ to $\pi(t^{(\kappa)})$ and thus $t^{(\kappa)} - t^{(\kappa-1)} \leq U, \forall \kappa \geq 1$.  
Therefore, $t^{(\mathcal K)} \leq U \cdot \sum_{j=1}^m K_j$. %The upper bound for both high and low level search is $\mathcal O(|V|^3 \cdot \sum_{j=1}^m K_j)$.
CBS-TA-MLA can thus safely prune any state whose time is larger than $U \cdot \sum_{j=1}^m K_j$ on the low level and terminate on Line 35 for any given unsolvable MG-TAPF instance when OPEN eventually becomes empty in finite time. 
\end{proof}

\subsection{Example}
Consider the example in Figure \ref{fig:example_map} with two agents $a_1$ and $a_2$ located at $s_1=B1$ and $s_2=A2$ respectively. Two tasks $\boldsymbol g_a$ and $\boldsymbol g_b$ will be assigned to two agents, where $\boldsymbol g_a[1]=C2$, $\boldsymbol g_a[2]=A2$, $\boldsymbol g_b[1]=B3$ and $\boldsymbol g_b[2]=B1$. The corresponding high-level forest of CBS-TA-MLA is shown in Figure \ref{fig:example_forest}. The first root CT node $N_{1}$ assigns $\boldsymbol g_b$ to $a_1$ and $\boldsymbol g_a$ to $a_2$.  CBS-TA-MLA chooses to expand the node with the minimum cost, which is $N_{1}$. It detects two collisions $\langle a_1, a_2, B2, 1 \rangle$ 
and $\langle a_1, a_2, B2, 3 \rangle$. Since $N_1$ is a root CT node, the second root node $N_2$ with next-best task assignment $\boldsymbol g_a$ to $a_1$ and $\boldsymbol g_b$ to $a_2$ is created and added to the OPEN list. CBS-TA-MLA resolves the earliest collision $\langle a_1, a_2, B2, 1 \rangle$ by creating two child nodes $N_3$ and $N_4$, where $a_1$ is prohibited from being in location $B2$ at time 1 in $N_3$ by adding $\langle a_1, B2, 1 \rangle$ to $N_3.constraint$ and $a_2$ is prohibited from being in location $B2$ at time 1 in $N_4$. As the low-level search can find paths for the replanned agents, both child nodes are added to the OPEN list. In the next iteration, CBS-TA-MLA picks $N_2$ for expansion, but does not create a root node since there are only two possible task assignments. The paths in $N_2$ have collisions, and thus CBS-TA-MLA generates two child nodes and adds them to the OPEN list. Then, it selects $N_3$ which has no collisions. So, it declares $N_3$ the goal node and returns the task assignment and paths. 

\begin{figure}[t]
    \centering
    \includegraphics[width=.45\columnwidth]{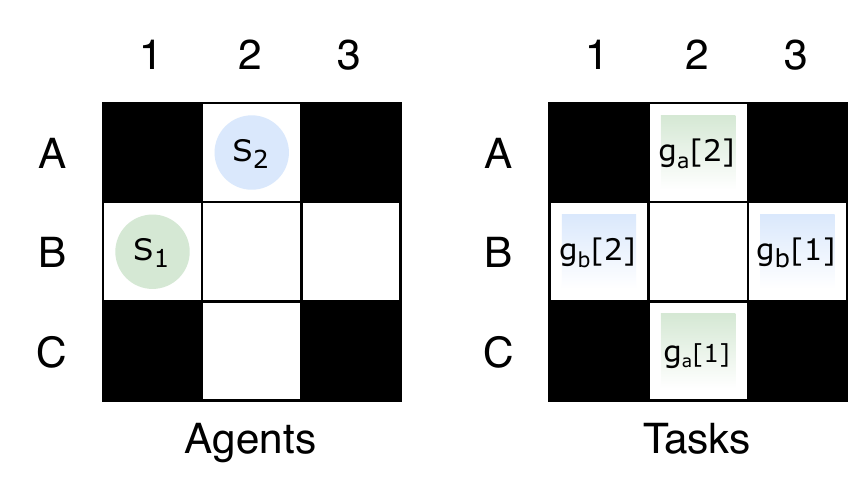}
    \caption{An example instance with agents and tasks.}
    \label{fig:example_map}
\end{figure}

\begin{figure}[t]
    \centering
    \includegraphics[width=.5\columnwidth]{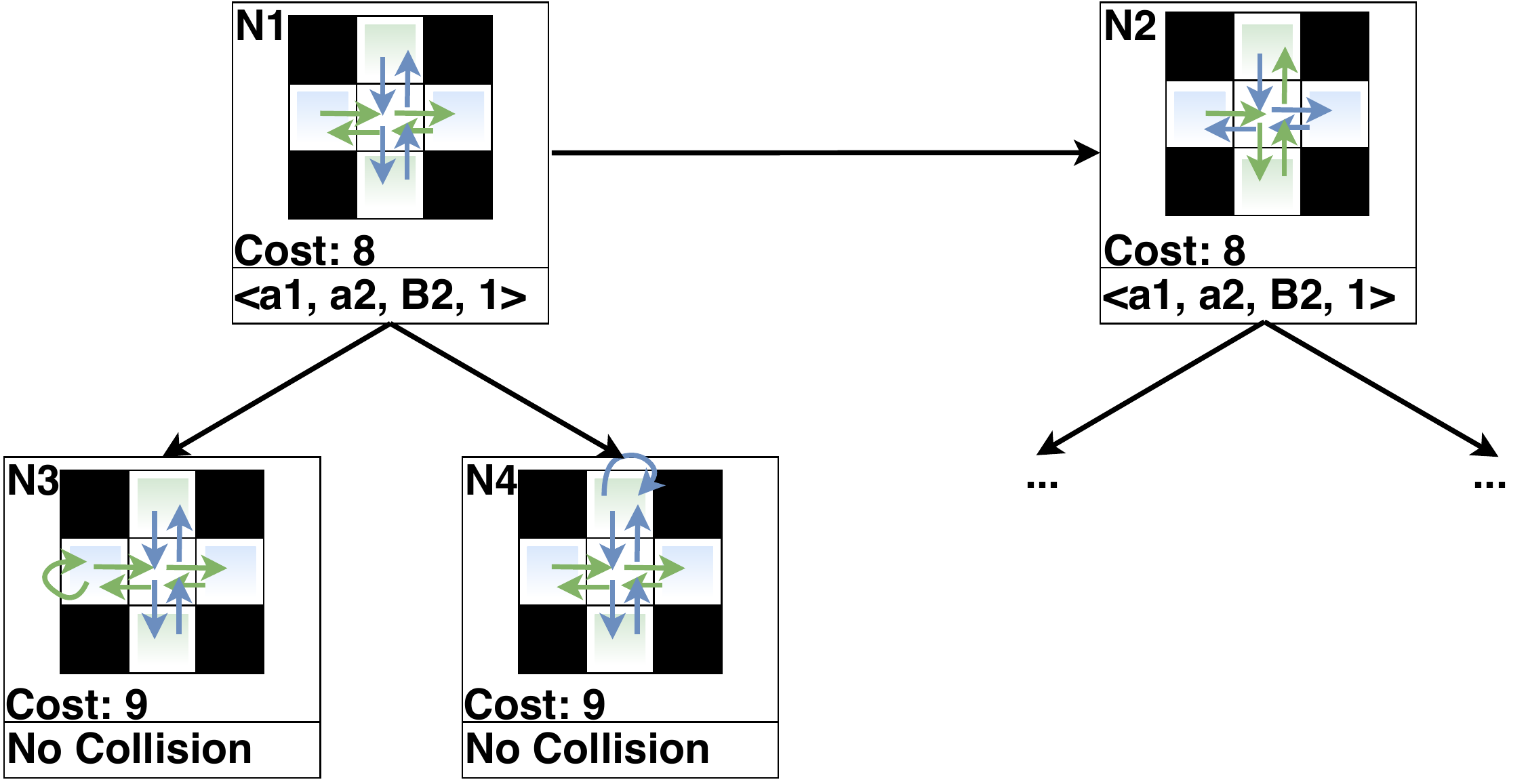}
    \caption{The search forest created by the high-level search of CBS-TA-MLA for the example from Figure \ref{fig:example_map}. }
    % The earliest collision among the paths is shown at the bottom of each node. }
    \label{fig:example_forest}
\end{figure}

%%%%%%%%%%%%%%%%%%%%%%%%%%%%%%%%%%%%%%%%%%%%%%%%%%%%%%%%%%%%%%%%%%%%%%%%%%%%%%%%

\section{Extensions of CBS-TA-MLA}

This section introduces three extensions of CBS-TA-MLA, namely an improved optimal version (with heuristics), a bounded-suboptimal version and a greedy version.

\subsection{CBS-TA-MLA with Heuristics (CBSH-TA-MLA)}
CBS with heuristics \cite{LiIJCAI19} introduces three admissible heuristics (namely CG, DG and WDG) for the high-level search of CBS, which reduce the number of expanded CT nodes. The collisions among paths of a CT node are classified in three types\cite{ICBS}: cardinal collisions if both of the resulting child nodes have a larger cost than the node itself, semi-cardinal collisions if only one of its child nodes has a larger cost than the node itself, and non-cardinal collisions if both of the child nodes have the same cost as the node itself. 

The technique to classify collisions is a Multi-Valued Decision Diagrams (MDDs) \cite{ICBS}. An MDD for agent $a_i$ at CT node $N$ is a directed acyclic graph consisting of all possible cost-optimal paths of $a_i$ with respect to the task assignment and constraints of $N$. Each MDD node consists of a location $v$ and a level/time $t$. A collision between agents $a_i$ and $a_j$ at time $t$ is cardinal iff the contested vertex/edge is the only vertex/edge at level $t$ of the MDDs of both agents ($t=1$ in Figure 3). To make the MDDs applicable in our case in which the cost-minimal paths contain all goal locations of the assigned task in the correct sequence, we add a label to each MDD node (see Figure \ref{fig:example_mdd}). 

CG heuristic only considers the cardinal collisions among paths. DG heuristic considers the dependency among agents and WDG heuristic considers the extra cost that each pair of dependent agents contributes to the total cost. WDG dominates DG, which dominates CG. See \cite{LiIJCAI19} for details. 

We adopt the techniques of CBS-TA-MLA for the CBS-TA-MLA with Heuristics algorithm (CBSH-TA-MLA). We maintain a new variable $min\_f\_val$ for the minimum $f$-value of all nodes in the OPEN list. Each node $N$ has two additional fields, namely $N.h\_val$ to represent the admissible $h$-value and $N.f\_val=N.cost+N.h\_val$ to represent the priority in the OPEN list. See \cite{LiIJCAI19} for details of $N.h\_val$ computation method. 

The chooseCollision(N) function in Algorithm \ref{alg:CBS-TA} chooses cardinal collisions first, semi-cardinal collisions next and non-cardinal collisions last (breaking ties in favor of the earliest collision). 

\begin{figure}[t]
\centering
\includegraphics[width=.5\columnwidth]{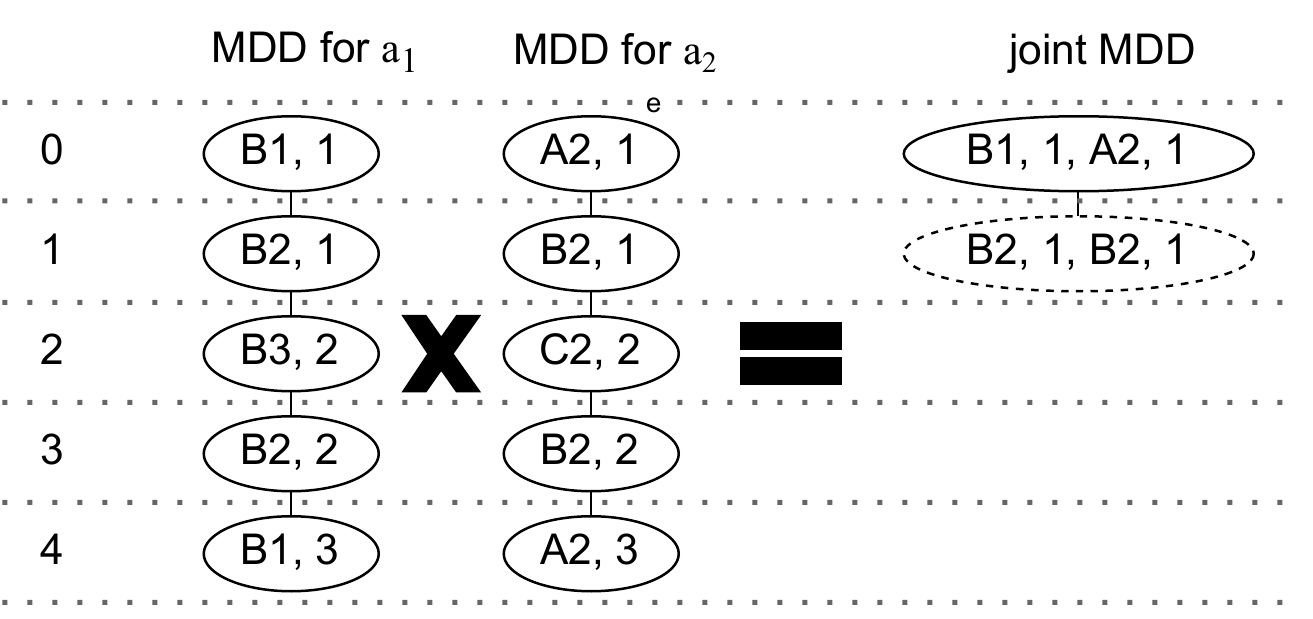} 
\caption{The MDDs and joint MDD for the example instance in Figure \ref{fig:example_map}. Levels of MDD nodes are shown on the left.}
\label{fig:example_mdd}
\end{figure}

\subsection{Enhanced CBS-TA-MLA (ECBS-TA-MLA)}

Similar to Enhanced CBS (ECBS) \cite{ECBS}, ECBS-TA-MLA is a bounded-suboptimal algorithm for MG-TAPF. It uses a focal search %instead of a best-first search 
on both high and low levels. A focal search maintains two lists: OPEN and FOCAL. The OPEN list is sorted in increasing order of the $f$-values. The best node $N_{best}$ in the OPEN list is the node with the minimum $f$-value, which is denoted by $f_{best}$. The FOCAL list contains that subset of the nodes in the OPEN list whose $f$-values are at most $\omega \cdot f_{best}$.
The FOCAL list is sorted according to some other heuristic. The FOCAL search guarantees to find solutions that are a factor of at most $\omega$ worse than optimal by always expanding the best node in the FOCAL list.

\noindent \textbf{Low-level focal search: } The low-level focal search prioritizes nodes in the OPEN list with $f$-values and nodes in the FOCAL list with the number of collisions in paths between the current agent $a_i$ and the other agents in the CT node. When it finds a solution, it returns the path and the $f$-value of the best node $n$ in the OPEN list, which is the lower bound on the cost of the time-optimal path, denoted by $f_{best}(a_i)$. 

\noindent \textbf{High-level focal search: }
The high-level focal search sorts CT nodes in the OPEN list in increasing order of the sum of the lower bounds of all agents $LB(N)=\sum_{i=1}^m f_{best}(a_i)$. Let $N_{best}$ denote the node $N$ in the OPEN list with the minimum $LB(N)$. The FOCAL list contains that subset of CT nodes $N$ with $N.cost \leq \omega \cdot LB(N_{best})$. The nodes in the FOCAL list are sorted in increasing order of the number of collisions among $N.paths$. Since $LB(N_{best})$ is provably a lower bound on the optimal flowtime $C^*$, the cost of any CT node in the FOCAL list is at most $\omega \cdot C^*$. As a result, once a solution is found, its flowtime is at most $\omega \cdot C^*$, so it is bounded-suboptimal with suboptimality factor $\omega$.

\subsection{Greedy CBS-TA-MLA (TA+CBS-MLA)}

The %greedy version of CBS-TA-MLA, called 
TA+CBS-MLA performs best task assignment (TA) followed by the CBS-MLA algorithm. It starts with the root node with the best task assignment and does not generate any other root nodes. TA+CBS-MLA provides no optimality or completeness guarantee.  

%%%%%%%%%%%%%%%%%%%%%%%%%%%%%%%%%%%%%%%%%%%%%%%%%%%%%%%%%%%%%%%%%%%%%%%%%%%%%%%%

\section{Experiments}

This section describes our experimental results on a 2.3GHz Intel Core i5 laptop with 16GB RAM. The algorithms are implemented in Python and tested on three maps, namely (1) a dense map, which is a $20\times 20$ warehouse map with $30\%$ obstacles (Figure \ref{fig:map_20x20}), (2) a sparse map, which is a $32\times 32$ random map with $10\%$ obstacles (Figure \ref{fig:map_32x32}), and (3) a $32\times 32$ empty map, all with a time limit of 120 seconds unless otherwise specified. 

\begin{figure}[t!]
\centering
    \begin{subfigure}[b]{0.48\columnwidth}\centering
    \includegraphics[width=.4\textwidth]{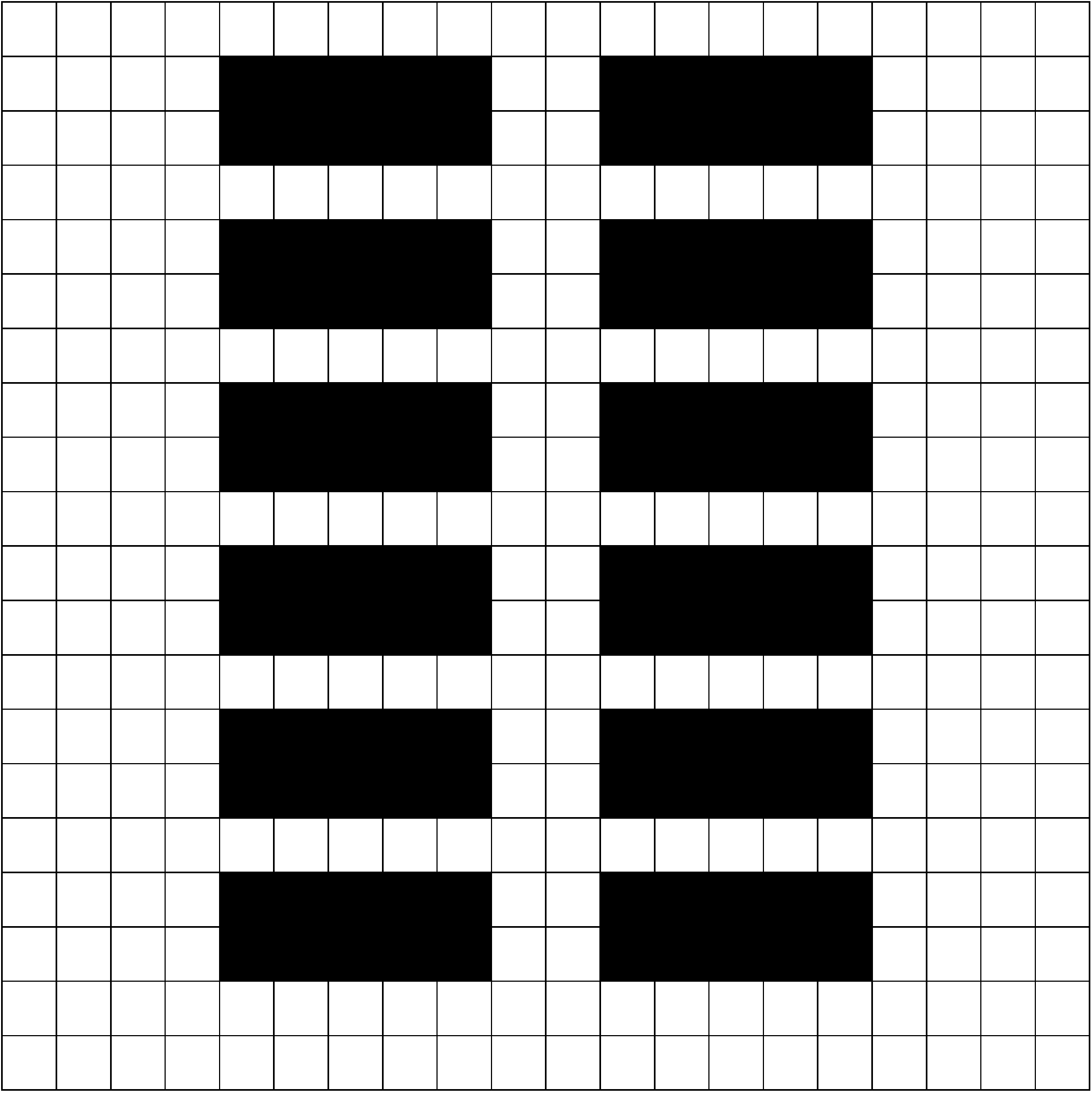}
    \caption{%$20\times 20$ warehouse map with $30\%$ obstacles. 
    Dense map}\label{fig:map_20x20}
    \end{subfigure}
    \hfill
    \begin{subfigure}[b]{0.48\columnwidth}\centering
    \includegraphics[width=.4\textwidth]{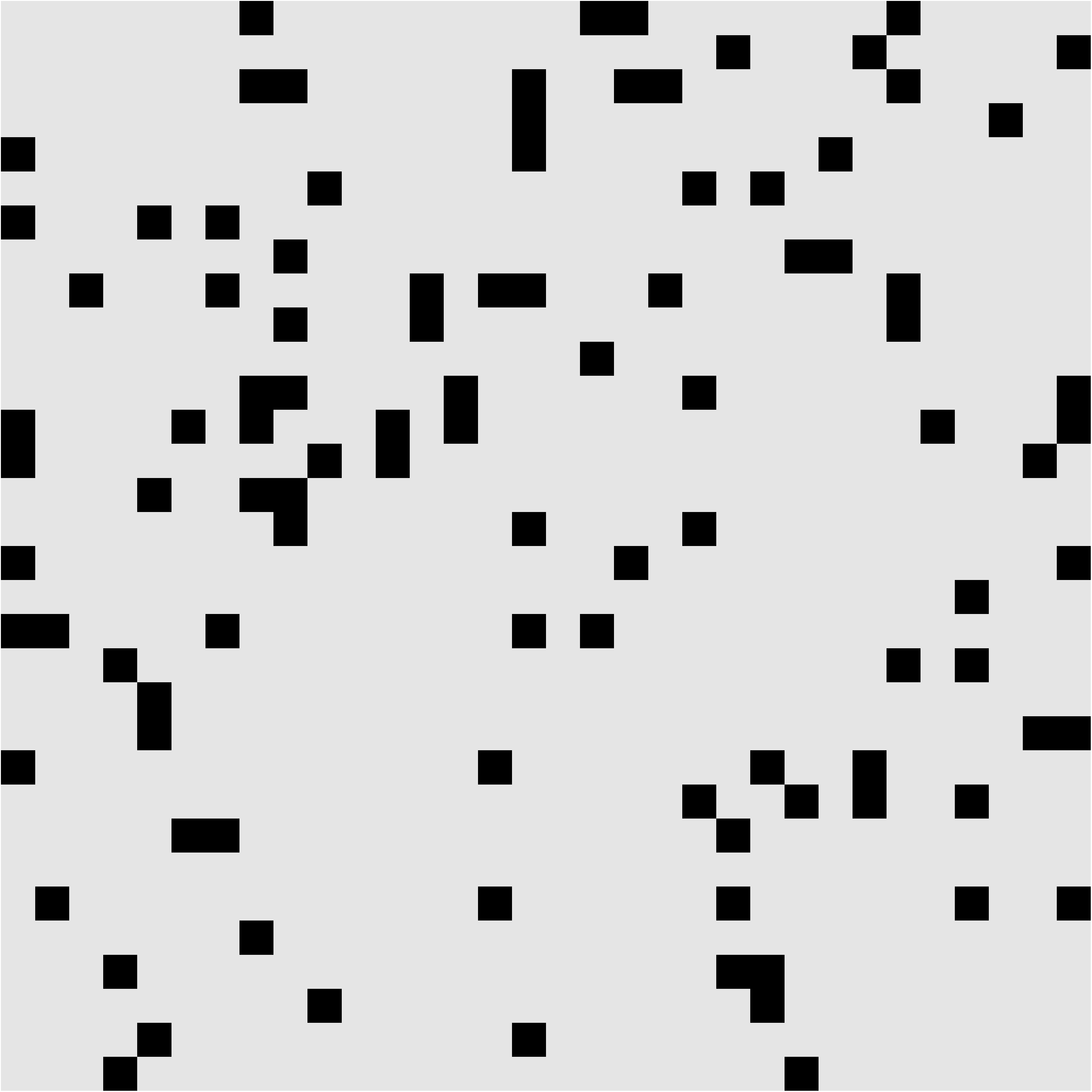}
    \caption{%$32\times 32$ random map with $10\%$ obstacles. 
    Sparse map}\label{fig:map_32x32}
    \end{subfigure}
    
\caption{Two maps used for the experiments.}
\end{figure}

\subsection{CBSH-TA-MLA}
We evaluate CBSH-TA-MLA using two test sets. In the first set, we use the dense map with 10 agents/tasks with randomly generated locations and report the success rate, the average number of expanded CT nodes and the average runtime over 100 instances. In the second test set, we use the sparse map and the empty map and report the above three values over 50 instances with a time limit of 300 seconds. The last two values are averaged over instances that are successfully solved by CBSH-TA-MLA with all four heuristics. Table \ref{tab:cbs_w_h} shows that WDG always results in the smallest number of expanded nodes, while DG always results in the smallest average runtime on all instances. This is so because WDG needs to compute the weights of the edges of the pairwise dependency graph, which requires executing the CBS-TA-MLA algorithm for two agents repeatedly.
%Although we set a small time limit of 5 seconds for the two-agent CBS-TA-MLA algorithm, it still requires a significant amount of time.

\begin{table}[t]%[th!]
    \renewcommand{\arraystretch}{0.5}
    \caption{Results for CBSH-TA-MLA using different heuristics on different maps with different numbers of agents/tasks, where each task consists of two goal locations.}
    \label{tab:cbs_w_h}
    \centering
    \Huge
	%\resizebox{\columnwidth}{!}{%
	\scalebox{0.32}{%
    \begin{tabular}{|c|c|c||c|r|r|}
    \hline
     Map & Agents & Heuristics & \begin{tabular}[c]{@{}c@{}}Success\\Rate\end{tabular} & \begin{tabular}[c]{@{}c@{}}Nodes\\Expanded\end{tabular} & Runtime (s) \\
     \hline
     \multirow{4}{*}{Dense Map}&\multirow{4}{*}{10}&No&98/100& 34.18 & 2.54 \\
     &&CG&98/100& 28.86 &  2.23   \\
     &&DG&98/100& 25.54 &  \textbf{2.09}   \\
     &&WDG&97/100& \textbf{8.98} &  3.53 \\
    \hline 
    
    \multirow{4}{*}{Sparse Map}&\multirow{4}{*}{20}&No&44/50& 42.59 & 15.68 \\
     &&CG&44/50& 34.65 &  13.21   \\
     &&DG&46/50& 30.58 &   \textbf{11.8}  \\
     &&WDG&46/50& \textbf{4.63} & 13.17  \\
     \hline
    
    \multirow{4}{*}{Empty Map}&\multirow{4}{*}{20}&No&46/50& 23.02 & 6.96 \\
     &&CG&46/50& 12.08 & 4.03   \\
     &&DG&50/50& 3.48 &  \textbf{1.52}   \\
     &&WDG&48/50& \textbf{2.11} & 2.82  \\
     \hline
     
    \multirow{4}{*}{Empty Map}&\multirow{4}{*}{30}&No&40/50& 20.225 & 9.55  \\
     &&CG&40/50& 14.75 & 7.15    \\
     &&DG&46/50& 7.375 &  \textbf{4.24}   \\
     &&WDG&46/50& \textbf{4.375} &  4.52 \\
     \hline
     
    \end{tabular}
    }
\end{table}

\begin{table}[t!]
    \renewcommand{\arraystretch}{0.5}
    \caption{Results for ECBS-TA-MLA with different $\omega$ on different maps with different numbers of agents/tasks and different numbers of goal locations per task.}
    \label{tab:ecbs_ta}
    \centering
    \Huge
    %\resizebox{\columnwidth}{!}{%
	\scalebox{0.32}{%
    \begin{tabular}{|c|c|c|c||c|r|r|r|}
    \hline
     Map & Agents & \begin{tabular}[c]{@{}c@{}}Goal\\Locations\end{tabular}& $\omega$ & \begin{tabular}[c]{@{}c@{}}Success\\Rate\end{tabular} & \begin{tabular}[c]{@{}c@{}}Nodes\\ Expanded\end{tabular} & Runtime (s) & Cost \\
    \hline
     \multirow{12}{*}{\begin{tabular}[c]{@{}c@{}}Dense\\Map\end{tabular}}&\multirow{4}{*}{2}&\multirow{4}{*}{2}&1.00&100/100& 22.68 & 2.20 & \textbf{144.69} \\
     &&&1.05&100/100& 6.00 &  0.70 & 145.57  \\
     &&&1.10&100/100& 3.27 &  0.32 & 146.77   \\
     &&&1.30&100/100& \textbf{0.85} &  \textbf{0.08} & 148.19 \\
    \cline{2-8}
    
    &\multirow{4}{*}{20}&\multirow{4}{*}{2}&1.00&29/100& 247.46 & 28.47 & \textbf{249.46} \\
     &&&1.05&78/100& 20.42 &  2.22 & 255.75  \\
     &&&1.10&97/100& 6.68 &  0.83 & 260.21   \\
     &&&1.30&\textbf{100/100}& \textbf{3.46} & \textbf{0.58} & 262.32 \\
    \cline{2-8}
    
    &\multirow{4}{*}{30}&\multirow{4}{*}{2}&1.00&0/100& / & / & / \\
     &&&1.05&14/100& / &  / & /  \\
     &&&1.10&48/100& / & / & /   \\
     &&&1.30&\textbf{96/100}& / &  /& / \\
    \hline
    
    \multirow{12}{*}{\begin{tabular}[c]{@{}c@{}}Sparse\\Map\end{tabular}}&\multirow{4}{*}{10}&\multirow{4}{*}{2}&1.00&100/100& 3.14 & 0.58 & \textbf{304.79} \\
     &&&1.05&100/100& 0.47 &  0.12 & 305.64  \\
     &&&1.10&100/100& 0.36 &  0.11 & 306.32   \\
     &&&1.30&100/100& \textbf{0.35} &  \textbf{0.10} & 306.79 \\
    \cline{2-8}
    
    &\multirow{4}{*}{20}&\multirow{4}{*}{2}&1.0&76/100& 28.34 & 9.22 & \textbf{537.49} \\
     &&&1.05&100/100& 1.96 &  0.79 & 541.38  \\
     &&&1.1&100/100& 1.48 &  \textbf{0.69} & 543.28   \\
     &&&1.3&100/100& \textbf{1.35} &  0.73 & 544.36 \\
    \cline{2-8}
    
   &\multirow{4}{*}{10}&\multirow{4}{*}{10}&1.00&73/100& 9.33 & 13.35 & \textbf{1846.46} \\
     &&&1.05&\textbf{88/100}& 0.68 &  4.52 & 1856.49  \\
     &&&1.10&86/100& \textbf{0.65} & 3.56 & 1858.03   \\
     &&&1.30&85/100& \textbf{0.65} &  \textbf{0.75} & 1858.03 \\
    \hline
    
    \end{tabular}
    }
\end{table}

\begin{table}[t!]
    \renewcommand{\arraystretch}{0.5}
    \caption{Results for different numbers of goal locations per task for ECBS-TA-MLA with different $\omega$ on the sparse map. The number of goal locations differs in different tasks.}
    \label{tab:num_goals}
    \centering
    \Huge
    %\resizebox{\columnwidth}{!}{%
	\scalebox{0.32}{%   
    \begin{tabular}{|c|c|c|r|r|}
    \hline
        Goal Locations & $\omega$ & Success Rate & Nodes Expanded & Runtime (s) \\
        \hline
        \multirow{2}{*}{2-5} & 1.1 & 95/100 & 1.26 & 1.96\\
        & 1.3 & 95/100 & 1.26 & 2.13\\
        \hline
        \multirow{2}{*}{6-10} & 1.1 & 72/100 & 2.04 & 3.51\\
        & 1.3 & 72/100 & 2.00 & 3.61\\
        \hline
        \multirow{2}{*}{11-15} & 1.1 & 80/100 & 7.03 & 18.36\\
        & 1.3 & 83/100 & 7.20 & 19.46\\
        \hline
        \multirow{2}{*}{16-20} & 1.1 & 78/100 & 14.61 & 48.11\\
        & 1.3 & 83/100 & 15.15 & 51.37\\
    \hline
    \end{tabular}
    }
\end{table}

\subsection{ECBS-TA-MLA}
We evaluate ECBS-TA-MLA with different suboptimality factors $\omega$ on the dense and the sparse maps with different numbers of agents/tasks (for both maps) and different numbers of goal locations in each task (for the sparse map). We report the success rate, the average number of expanded CT nodes, the average runtime and the average cost over 100 instances in Table \ref{tab:ecbs_ta}. The last three values are averaged over instances that are successfully solved by ECBS-TA-MLA with all four $\omega$. As expected, ECBS-TA-MLA achieves high success rates on the instances with small numbers of agents/tasks. The success rates only drop a bit with $\omega=1.3$ when the number of agents increases up to 30. In addition, when we increase $\omega$, the average number of expanded CT nodes and the average runtime decrease, while the average cost increases. %We vary the number of agents/tasks and the number of goal locations per task (namely, 2 goal locations per task and 10 goal locations per task).% When the number of goal locations per task is 10, the success rate decreases as $\omega$ increases from 1.05 to 1.30.

We test ECBS-TA-MLA when tasks have different number of goal locations with 10 agents/tasks and report the same values as in Table \ref{tab:cbs_w_h} in Table \ref{tab:num_goals}. The experiment shows that, with an appropriate $\omega$, the success rate is still over $70\%$ within two minutes, even with a maximum of 20 goal locations per task. 

\iffalse
Figure \ref{fig:ecbs} shows the ratio of the solution cost to the cost of the first root node, with a large number of agents/tasks $m=40, 60, 80$. The ratio increases as the number of agents/tasks increase. In general, the ratio is the largest when $\omega=2$, but it is not guaranteed. This test shows how $\omega$ affects the plan cost with different numbers of agents/tasks.

\begin{figure*}[t!]
\centering
\includegraphics[width=\textwidth]{ecbs.pdf} 
\caption{The ratio of the real cost and the cost of first root node in the high-level search of ECBS-TA-MLA where the number of agents is 40 (left), 60 (middle) and 80 (right). }
\label{fig:ecbs}
\end{figure*}
\fi

\subsection{TA+CBS-MLA}

We compare TA+CBS-MLA against CBS-TA-MLA on the dense map and report the success rate, the average cost, and the average runtime with 10 agents/tasks, where each task consists of two goal locations,  %and two goal locations per task 
over 100 instances in Table \ref{tab:tacbs_table}. CBS-TA-MLA outperforms TA+CBS-MLA in the average cost, but TA+CBS-MLA outperforms CBS-TA-MLA in both the success rate and the average runtime. The reason for this is that the number of possible task assignments is large, namely $10! \approx 3$ millions, so CBS-TA-MLA spends a significant amount of time computing task assignments.  

\begin{table}[t!]
    \renewcommand{\arraystretch}{0.5}
    \caption{Results for TA+CBS-MLA and CBS-TA-MLA on the dense map with 10 agents/tasks. }
    \label{tab:tacbs_table}
    \centering
    \Huge
    \scalebox{0.32}{%
    \begin{tabular}{|c|c|r|r|}
    \hline  
        & Success Rate & Runtime (s) & Cost \\
        \hline 
        TA+CBS-MLA & \textbf{100/100} & \textbf{1.92} & 148.16\\
        \hline
        CBS-TA-MLA & 99/100 & 3.04 & \textbf{146.63}\\
    \hline 
    \end{tabular}
    }
\end{table}

\iffalse
\begin{table}[t!]
    \renewcommand{\arraystretch}{0.5}
    \caption{Results for TA+CBS-MLA and CBS-TA-MLA on the dense map with 10 agents/tasks. }
    \label{tab:tacbs_table}
    \centering
    \Huge
    %\resizebox{\columnwidth}{!}{%
    \scalebox{0.32}{%
    \begin{tabular}{|c|c|c|}
    \hline  
         & TA+CBS-MLA & CBS-TA-MLA \\
         \hline 
        Success Rate & 100\% & 99\%\\
        \hline
        Average Runtimes & 1.92 & 3.04\\
        \hline
        Average Costs & 148.16 & 146.63 \\
        %\hline
        %Runtime Saving & \multicolumn{2}{c|}{-1.1191} \\
        %\hline
        %Cost Saving &  \multicolumn{2}{c|}{1.52}\\
    \hline 
    \end{tabular}
    }
\end{table}
\fi
%%%%%%%%%%%%%%%%%%%%%%%%%%%%%%%%%%%%%%%%%%%%%%%%%%%%%%%%%%%%%%%%%%%%%%%%%%%%%%%%

\section{Conclusion and future work}

In this paper, we presented the CBS-TA-MLA algorithm to solve the MG-TAPF problem optimally. We presented three enhanced variants of CBS-TA-MLA, namely (1) optimal variant CBSH-TA-MLA, which speeds up CBS-TA-MLA by adding a heuristic, (2) bounded-suboptimal variant ECBS-TA-MLA, which speeds up CBS-TA-MLA by using focal search and (3) greedy variant TA+CBS-MLA, which commits to the most promising task assignment without exploring other assignments. We conducted experiments to evaluate these algorithms in different settings. %with different maps, different numbers of agents/tasks, and different numbers of goal locations per task.
It is future research to incorporate additional enhancements (such as disjoint splitting for the high-level search and incremental A* for the low-level search) into CBS-TA-MLA to improve its efficiency.

\section*{ACKNOWLEDGMENT}
The research at Simon Fraser University was supported by the Natural Sciences and Engineering Research Council of Canada under grant number RGPIN2020-06540. The research at the University of Southern California was supported by the National Science Foundation under grant numbers 1409987, 1724392, 1817189, 1837779, 1935712, and 2112533 as well as a gift from Amazon.

\bibliographystyle{IEEEtran}
\bibliography{IEEEfull,sample.bib}

\end{document}